\documentclass[11pt]{article}
\usepackage[a4paper,margin=1in]{geometry}
\usepackage[T1]{fontenc}
\usepackage{newtxtext}
\usepackage{amsmath,amssymb,amsthm}
\usepackage{newtxmath}
\usepackage{microtype}
\usepackage{mathtools}
\usepackage{booktabs}
\usepackage{multirow}
\usepackage{graphicx}
\usepackage{caption}
\usepackage{subcaption}
\usepackage{enumitem}
\usepackage{xcolor}
\usepackage{float}
\usepackage[section]{placeins}
\usepackage{bm}
\usepackage{hyperref}
\usepackage{tikz}
\usepackage{tikz-3dplot}
\usetikzlibrary{shapes,arrows}

\hypersetup{colorlinks=true,linkcolor=blue,citecolor=blue,urlcolor=blue}

\usepackage[capitalize,noabbrev]{cleveref}

\newtheorem{theorem}{Theorem}[section]
\newtheorem{proposition}[theorem]{Proposition}
\newtheorem{lemma}[theorem]{Lemma}
\newtheorem{corollary}[theorem]{Corollary}
\theoremstyle{definition}

\theoremstyle{remark}
\newtheorem{remark}[theorem]{Remark}

\newcommand{\E}{\mathbb{E}}

\newcommand{\KL}{\mathrm{KL}}

\newcommand{\loss}{\mathcal{L}}

\newcommand{\piref}{\pi_{\mathrm{ref}}}
\newcommand{\pik}{\pi_k}

\title{\textbf{Group Orthogonalized Policy Optimization}\\
\large Group Policy Optimization as Orthogonal Projection in Hilbert Space}

\author{Wang Zixian\\
\small China Mobile Communications Group Shandong Co., Ltd. Tai'an Branch\\
\small \texttt{wangzixian@sd.chinamobile.com}}
\date{}

\begin{document}
\maketitle

\begin{abstract}
We present \textbf{Group Orthogonalized Policy Optimization (GOPO)}, a new alignment algorithm for large language models derived entirely from the geometry of Hilbert function spaces. Unlike existing methods that formulate alignment as constrained optimization on the probability simplex---inevitably inheriting the exponential curvature of Kullback--Leibler divergence---we propose a fundamentally different approach: lifting the problem into the Hilbert space $L^2(\pik)$ of square-integrable functions with respect to the current reference policy.

In this function space, we make three observations that jointly yield a principled algorithm. First, the probability simplex constraint reduces to a single linear orthogonality condition $\langle v, \mathbf{1} \rangle = 0$, defining a codimension-one closed subspace $\mathcal{H}_0$. Second, the alignment objective emerges naturally from the \emph{Geometric Principle of Minimum Distance}: minimizing the geometric distance between $v$ and the unconstrained target $u^*$ algebraically yields the \emph{work-dissipation functional} $\mathcal{J}(v) = \langle g, v \rangle - \frac{\mu}{2}\|v\|^2$. The global maximum of this functional is then given elegantly by the Hilbert Projection Theorem. Third, enforcing the non-negativity boundary $v \geq -1$ upgrades the solution to a \emph{Bounded Hilbert Projection} (BHP), yielding exact sparsity: catastrophically poor actions receive zero target probability via a hard analytical threshold.

To bridge the gap between this global functional theory and practical training, GOPO transitions the projection from the infinite-dimensional $L^2(\pik)$ to a finite-dimensional empirical subspace induced by group sampling. A key structural result emerges: because group-normalized advantages inherently sum to zero, the Lagrange multiplier (chemical potential) enforcing probability conservation \emph{exactly vanishes}, collapsing the constrained projection into an unconstrained empirical loss. The resulting objective features constant Hessian curvature $\mu I$, non-saturating linear gradients, and an intrinsic dead-zone mechanism---all without heuristic clipping. Experiments on mathematical reasoning benchmarks demonstrate that GOPO achieves competitive generalization while sustaining healthy gradient dynamics and entropy preservation in regimes where clipping-based methods plateau.
\end{abstract}

\section{Introduction}
\label{sec:intro}

Aligning Large Language Models (LLMs) with human preferences via Reinforcement Learning from Human Feedback (RLHF)~\cite{Christiano2017RLHF,Ouyang2022InstructGPT} has become a standard practice. Dominant algorithms---Proximal Policy Optimization (PPO)~\cite{Schulman2017PPO}, Direct Preference Optimization (DPO)~\cite{Rafailov2023DPO}, and Group Relative Policy Optimization (GRPO)~\cite{Shao2024GRPO}---all share a common structural element: they regularize policy updates via the Kullback--Leibler (KL) divergence, either explicitly as a penalty term or implicitly through ratio clipping in log-probability space.

KL-regularized objectives carry an inherent geometric limitation. The KL divergence $D_{\KL}(\pi_\theta \| \piref)$ induces an \emph{exponential} geometry: its Hessian in log-ratio coordinates scales as $\sigma(m)(1-\sigma(m))$, where $m$ is the logit margin. This data-dependent curvature means that as the policy becomes more confident---producing large log-ratios---the gradient signal decays exponentially. The resulting \emph{gradient saturation} forces practitioners to either accept premature plateaus or resort to heuristic mechanisms such as aggressive ratio clipping~\cite{Schulman2017PPO}, asymmetric clip bounds~\cite{Yu2025DAPO}, or entropy bonuses.

We argue that this limitation is not intrinsic to alignment but rather to the \emph{choice of geometric space} in which optimization is performed. The probability simplex, equipped with KL geometry, conflates two independent design choices: \emph{what to learn} (the sampling/advantage signal) and \emph{how to regularize} (the optimization curvature). Adjusting one inevitably perturbs the other.

\paragraph{A Hilbert Space Perspective.} In this work, we abandon the probability simplex altogether and reformulate alignment as an optimization problem in the Hilbert function space $L^2(\pik)$. This shift yields three structural advantages:

\begin{enumerate}[leftmargin=2em]
    \item The probability conservation constraint $\sum_y \pi(y) = 1$ becomes a \emph{linear} orthogonality condition $\langle v, \mathbf{1}\rangle_{\pik} = 0$, where $v(y) = \pi(y)/\pik(y) - 1$ is the density fluctuation field. This defines a closed linear subspace $\mathcal{H}_0 \subset L^2(\pik)$.
    \item The optimal policy update is obtained naturally by the \emph{Hilbert Projection Theorem}: the physically valid probability distribution that is geometrically closest to the unconstrained target. No heuristic clipping or ad-hoc normalization is needed; the classic work-dissipation functional emerges strictly as an algebraic consequence of this minimum distance.
    \item The Hessian of the resulting objective is the constant scalar $\mu I$, completely independent of the data distribution or the current policy state. This \emph{constant curvature} guarantee ensures that gradients remain linear and non-saturating throughout training.
\end{enumerate}

Building on the functional-analytic foundation established in Orthogonalized Policy Optimization (OPO)~\cite{Wang2025OPO}, we derive \textbf{Group Orthogonalized Policy Optimization (GOPO)}, which seamlessly transitions the orthogonal projection from the infinite-dimensional global space $L^2(\pik)$ to a finite-dimensional empirical subspace induced by group sampling. A particularly elegant result emerges at the group level: because standardized advantages structurally sum to zero, they already reside in the zero-mean subspace, and the chemical potential enforcing probability conservation \emph{exactly vanishes}. GOPO thus reduces to a simple, unconstrained quadratic loss over policy ratios.

\paragraph{Contributions.}
\begin{itemize}[leftmargin=2em]
    \item We formulate LLM alignment as a constrained optimization in $L^2(\pik)$ and show that the optimal policy is the orthogonal projection of the advantage-driven target onto the probability-conserving subspace.
    \item We extend the unconstrained projection to the \emph{Bounded Hilbert Projection} (BHP), which enforces non-negativity of the target policy and yields exact sparsity for catastrophically poor actions.
    \item We derive GOPO as the empirical realization of this framework, proving that the chemical potential vanishes under group normalization, producing a practical loss with constant curvature and an intrinsic dead-zone mechanism.
    \item Experiments on mathematical reasoning benchmarks demonstrate competitive generalization (47\% on MATH Level 4 vs.\ 44\% for GRPO/DAPO), monotonically improving validation accuracy, and the healthiest gradient dynamics among all tested methods.
\end{itemize}

\section{Related Work}
\label{sec:related}

\paragraph{Preference Optimization and RLHF.}
RLHF~\cite{Christiano2017RLHF,Ouyang2022InstructGPT} typically involves learning a reward model from preferences and optimizing a policy via PPO~\cite{Schulman2017PPO}. DPO~\cite{Rafailov2023DPO} simplifies this by deriving a closed-form solution to the KL-constrained problem. IPO~\cite{Azar2024IPO} adds a regularization term; SimPO~\cite{Meng2024SimPO} simplifies the reference-free objective. GRPO~\cite{Shao2024GRPO} eliminates the critic network by using group-relative advantages. All these methods inherit the exponential geometry of KL divergence, leading to gradient saturation in high-confidence regimes.

\paragraph{$f$-Divergences and Alternative Geometries.}
The $f$-divergence family~\cite{Csiszar1967,AliSilvey1966} provides a unified framework for distributional discrepancy. The Pearson $\chi^2$ divergence, a member of this family, induces quadratic rather than exponential penalties. Prior work has explored $f$-divergences in variational inference~\cite{Li2016Renyi}, GANs~\cite{Nowozin2016fGAN}, and imitation learning~\cite{Ghasemipour2020fIL}. In RL, $\alpha$PPO~\cite{Xu2023AlphaPPO} studied $\alpha$-divergence as a trust-region constraint. APO~\cite{zixian2025apo} explored combining forward and reverse KL dynamics. OPO~\cite{Wang2025OPO} first proposed lifting alignment into $L^2(\pik)$ and deriving the optimal update via the Hilbert Projection Theorem. GOPO extends this framework from the global Hilbert space to empirical group-level subspaces, proving that the chemical potential vanishes under group normalization and introducing the Bounded Hilbert Projection for exact sparsity.

\paragraph{Trust-Region Methods.}
TRPO~\cite{Schulman2015TRPO} enforces stability via explicit KL constraints. PPO~\cite{Schulman2017PPO} approximates this with ratio clipping. DAPO~\cite{Yu2025DAPO} introduces dynamic clip bounds and entropy management. ADPO~\cite{Wang2025ADPO} uses anchored coordinates for implicit trust regions. GOPO replaces all such mechanisms with a single quadratic penalty in ratio space, whose geometric justification is the Hilbert Projection Theorem rather than heuristic approximation.

\section{Theoretical Framework: Policy Optimization in Hilbert Space}
\label{sec:framework}

We develop the mathematical apparatus underlying GOPO. The central idea is to reformulate policy optimization in the Hilbert space $L^2(\pik)$, where the probability simplex constraint becomes a linear subspace condition and the optimal policy update is obtained by orthogonal projection.

\subsection{The Hilbert Space of Density Fluctuations}
\label{subsec:hilbert_space}

Let $\pik$ denote the current reference policy (typically the policy from the previous iteration under on-policy anchoring). We define the ambient space as the Hilbert space of square-integrable functions with respect to $\pik$:
\begin{equation}
    \mathcal{H} = L^2(\pik), \quad \langle f, g \rangle_{\pik} = \E_{\pik}[f(y)g(y)]
\end{equation}
with induced norm $\|f\|_{\pik} = \sqrt{\langle f, f \rangle_{\pik}}$.

Instead of directly optimizing the probability vector $\pi(y|x)$ on the simplex, we shift to a centered coordinate system. Define the \emph{density fluctuation field}:
\begin{equation}
\label{eq:fluctuation}
    v(y) = \frac{\pi(y)}{\pik(y)} - 1
\end{equation}
The target policy is exactly recovered as $\pi(y) = \pik(y)(1 + v(y))$. The field $v$ lives in $\mathcal{H}$ and represents the relative deviation of the target policy from the reference. This coordinate change is the key step that transforms the non-linear simplex geometry into a linear Hilbert space geometry.

\subsection{Probability Conservation as Orthogonality}
\label{subsec:conservation}

Any valid probability distribution must satisfy $\sum_y \pi(y) = 1$. In the fluctuation coordinate, this fundamental physical law simplifies dramatically:
\begin{equation}
    \sum_y \pi(y) = \sum_y \pik(y)(1 + v(y)) = 1 + \sum_y \pik(y) v(y) = 1 + \E_{\pik}[v] = 1
\end{equation}
Hence, probability conservation is equivalent to the zero-mean condition:
\begin{equation}
\label{eq:zero_mean}
    \E_{\pik}[v] = \langle v, \mathbf{1} \rangle_{\pik} = 0
\end{equation}
This is a \emph{single linear} constraint in $\mathcal{H}$, defining a closed subspace of codimension one:
\begin{equation}
\label{eq:H0}
    \mathcal{H}_0 = \{ f \in \mathcal{H} : \langle f, \mathbf{1} \rangle_{\pik} = 0 \}
\end{equation}

\begin{remark}[Geometric Interpretation]
$\mathcal{H}_0$ is precisely the orthogonal complement of the one-dimensional subspace spanned by the constant function $\mathbf{1}$ in $L^2(\pik)$. Every valid policy fluctuation must be \emph{orthogonal} to $\mathbf{1}$---probability mass gained at some outputs must be exactly compensated by mass lost elsewhere. The term ``orthogonalized'' in our method name refers directly to this geometric fact.
\end{remark}

\subsection{The Driving Force: Metric Modulation}
\label{subsec:driving_force}

In reinforcement learning, the optimization is driven by the advantage function $A(y)$. To incorporate advanced sampling strategies, we define a \emph{Metric Modulation Operator} $\mathcal{M}_\alpha$ that reshapes the raw advantage into an effective driving field.

Let $E_\alpha(y)$ be an $\alpha$-parameterized escort multiplier that reweights contributions according to the density ratio:
\begin{equation}
    E_\alpha(y) = \left(\frac{\pi(y)}{\pik(y)}\right)^\alpha
\end{equation}
The multiplication operator $\mathcal{M}_\alpha$ acts on functions in $\mathcal{H}$ as $(\mathcal{M}_\alpha f)(y) = E_\alpha(y) f(y)$. The \emph{effective driving field} is:
\begin{equation}
\label{eq:driving}
    g_\alpha(y) = E_\alpha(y) \cdot A(y)
\end{equation}
When $\alpha = 0$, the raw advantage drives the update directly. Nonzero $\alpha$ focuses the driving force toward regions where the policy already assigns significant mass, controlling the exploit-explore tradeoff independently of the optimization geometry.

\section{Methodology: Group Orthogonalized Policy Optimization}
\label{sec:method}

We now derive GOPO in four stages: (i) identifying the unconstrained target and formulating the geometric minimum-distance principle, (ii) showing that the classic work-dissipation functional emerges algebraically from this distance, (iii) obtaining the closed-form solution via the Hilbert Projection Theorem and its bounded extension, and (iv) transitioning to the empirical group-level objective.

\subsection{The Geometric Principle: Minimum Distance to the Unconstrained Target}
\label{subsec:geometric_principle}

Suppose for a moment that there were no physical constraints on the probability distribution. Driven entirely by the external metric-modulated field $g_\alpha$, the ideal, unconstrained policy fluctuation would simply align with the driving force, scaled by a compliance (or stiffness) parameter $\mu > 0$:
\begin{equation}
    u^* = \frac{g_\alpha}{\mu}
\end{equation}
However, the reality is that any valid fluctuation field $v$ must strictly obey probability conservation, locking it within the zero-mean hyperplane $\mathcal{H}_0$. The unconstrained target $u^*$ typically violates this constraint because $\E_{\pik}[g_\alpha] \neq 0$.

What is the most geometric and principled way to resolve this? We seek the valid fluctuation $v \in \mathcal{H}_0$ that is geometrically closest to the ideal target $u^*$ in the Hilbert space. Mathematically, this corresponds to minimizing the $L^2$ distance (scaled by $\mu/2$ for dimension consistency):
\begin{equation}
\label{eq:min_distance}
    \min_{v \in \mathcal{H}_0} \; \frac{\mu}{2} \left\| v - u^* \right\|^2_{\pik}
\end{equation}

\subsection{Emergence of the Work-Dissipation Functional}
\label{subsec:emergence}

By algebraically expanding this pure geometric distance, a remarkable physical structure naturally emerges. Substituting $u^* = g_\alpha/\mu$, we have:
\begin{align}
    \frac{\mu}{2} \left\| v - \frac{g_\alpha}{\mu} \right\|^2_{\pik} &= \frac{\mu}{2} \left( \|v\|^2_{\pik} - \frac{2}{\mu}\langle g_\alpha, v \rangle_{\pik} + \frac{1}{\mu^2}\|g_\alpha\|^2_{\pik} \right) \\
    &= \underbrace{\frac{\mu}{2} \|v\|^2_{\pik}}_{\text{Dissipation}} - \underbrace{\langle g_\alpha, v \rangle_{\pik}}_{\text{Work}} + \underbrace{\frac{1}{2\mu} \|g_\alpha\|^2_{\pik}}_{\text{Constant independent of } v}
\end{align}
Minimizing the geometric distance is therefore strictly equivalent to maximizing the negative of the $v$-dependent terms. Thus, the alignment objective becomes:
\begin{equation}
\label{eq:constrained_opt}
    \arg\min_{v \in \mathcal{H}_0} \frac{\mu}{2} \left\| v - u^* \right\|^2_{\pik} \equiv \arg\max_{v \in \mathcal{H}_0} \Big( \underbrace{\langle g_\alpha, v \rangle_{\pik}}_{\text{External Work}} - \underbrace{\frac{\mu}{2} \|v\|^2_{\pik}}_{\text{Quadratic Dissipation}} \Big)
\end{equation}
The classic ``work-dissipation'' functional $\mathcal{J}(v) = \langle g_\alpha, v \rangle_{\pik} - \frac{\mu}{2}\|v\|^2_{\pik}$ is not an arbitrarily constructed objective; it is the algebraic inevitability of measuring squared distance in a Hilbert space.

\subsection{The Orthogonal Projection Solution}
\label{subsec:projection}

Because the problem is fundamentally one of finding the closest point in a closed subspace $\mathcal{H}_0$ to a target vector $u^*$, we do not need to rely on Lagrangian differentiation. The Hilbert Projection Theorem directly yields the unique global optimum as the orthogonal projection:

\begin{theorem}[Optimal Fluctuation via Orthogonal Projection]
\label{thm:projection}
The optimal probability-conserving fluctuation $v^*$ is the orthogonal projection of the unconstrained target $u^*$ onto $\mathcal{H}_0$:
\begin{equation}
\label{eq:projection_solution}
    v^* = P_{\mathcal{H}_0}(u^*) = u^* - \langle u^*, \mathbf{1} \rangle_{\pik} \mathbf{1} = \frac{1}{\mu}\Big( g_\alpha - \E_{\pik}[g_\alpha] \cdot \mathbf{1} \Big)
\end{equation}
\end{theorem}

\begin{proof}
Since $\mathcal{H}_0 = \{\mathbf{1}\}^\perp$, the orthogonal projection $P_{\mathcal{H}_0}$ onto $\mathcal{H}_0$ is:
\begin{equation}
    P_{\mathcal{H}_0}(f) = f - \frac{\langle f, \mathbf{1} \rangle_{\pik}}{\|\mathbf{1}\|^2_{\pik}} \mathbf{1} = f - \E_{\pik}[f] \cdot \mathbf{1}
\end{equation}
using $\|\mathbf{1}\|^2_{\pik} = \E_{\pik}[1] = 1$. Applying this to $f = u^* = g_\alpha/\mu$ yields the result.
\end{proof}

\begin{remark}[The Chemical Potential]
\label{rem:chemical_potential}
The subtracted term $\lambda^* = \E_{\pik}[g_\alpha]/\mu$ is the Lagrange multiplier enforcing probability conservation. In our framework, it obtains the purest geometric interpretation: it is exactly the orthogonal component of the target vector $u^*$ along the illegal normal direction (the constant function $\mathbf{1}$) that is explicitly sliced off by the projection.
\end{remark}

\subsection{Bounded Hilbert Projection and Exact Sparsity}
\label{subsec:bhp}

The unconstrained projection~\eqref{eq:projection_solution} may produce $v^*(y) < -1$ for actions with strongly negative driving signals, yielding invalid negative probabilities. To enforce the physical constraint $\pi(y) \geq 0$, equivalently $v(y) \geq -1$, we extend the projection to a \emph{Bounded Hilbert Projection} (BHP).

Define the positive orthant $\mathcal{C}_+ = \{f \in \mathcal{H} : f(y) \geq -1, \forall y\}$ and the feasible convex set:
\begin{equation}
    \mathcal{K} = \mathcal{H}_0 \cap \mathcal{C}_+
\end{equation}
The BHP problem seeks the point in $\mathcal{K}$ closest to $g_\alpha/\mu$. We formulate the KKT conditions via the Lagrangian:
\begin{equation}
    \mathcal{L}(v, \lambda, \eta) = \E_{\pik}\!\left[ g_\alpha v - \frac{\mu}{2}v^2 - \lambda v + \eta(v + 1) \right]
\end{equation}
where $\lambda$ enforces $\E_{\pik}[v] = 0$ and the KKT multiplier $\eta(y) \geq 0$ enforces $v(y) \geq -1$.

\begin{theorem}[Bounded Projection Solution]
\label{thm:bhp}
The KKT complementarity condition $\eta(y)(v(y) + 1) = 0$ yields the closed-form bounded solution:
\begin{equation}
\label{eq:bhp_solution}
    v^*(y) = \max\!\left(-1, \; \frac{g_\alpha(y) - \lambda^*}{\mu}\right)
\end{equation}
where the dynamic chemical potential $\lambda^*$ is the unique value ensuring $\E_{\pik}[v^*] = 0$.
\end{theorem}

This result has a profound theoretical consequence:
\begin{corollary}[Exact Sparsity]
\label{cor:sparsity}
Actions satisfying $g_\alpha(y) < \lambda^* - \mu$ are assigned $v^*(y) = -1$, i.e., their target probability is \emph{exactly zero}: $\pi(y) = \pik(y)(1 + v^*(y)) = 0$. The BHP thus provides an analytical mechanism for eliminating catastrophically poor outputs (e.g., hallucinations) via a hard, data-adaptive threshold, without requiring any external filtering heuristic.
\end{corollary}

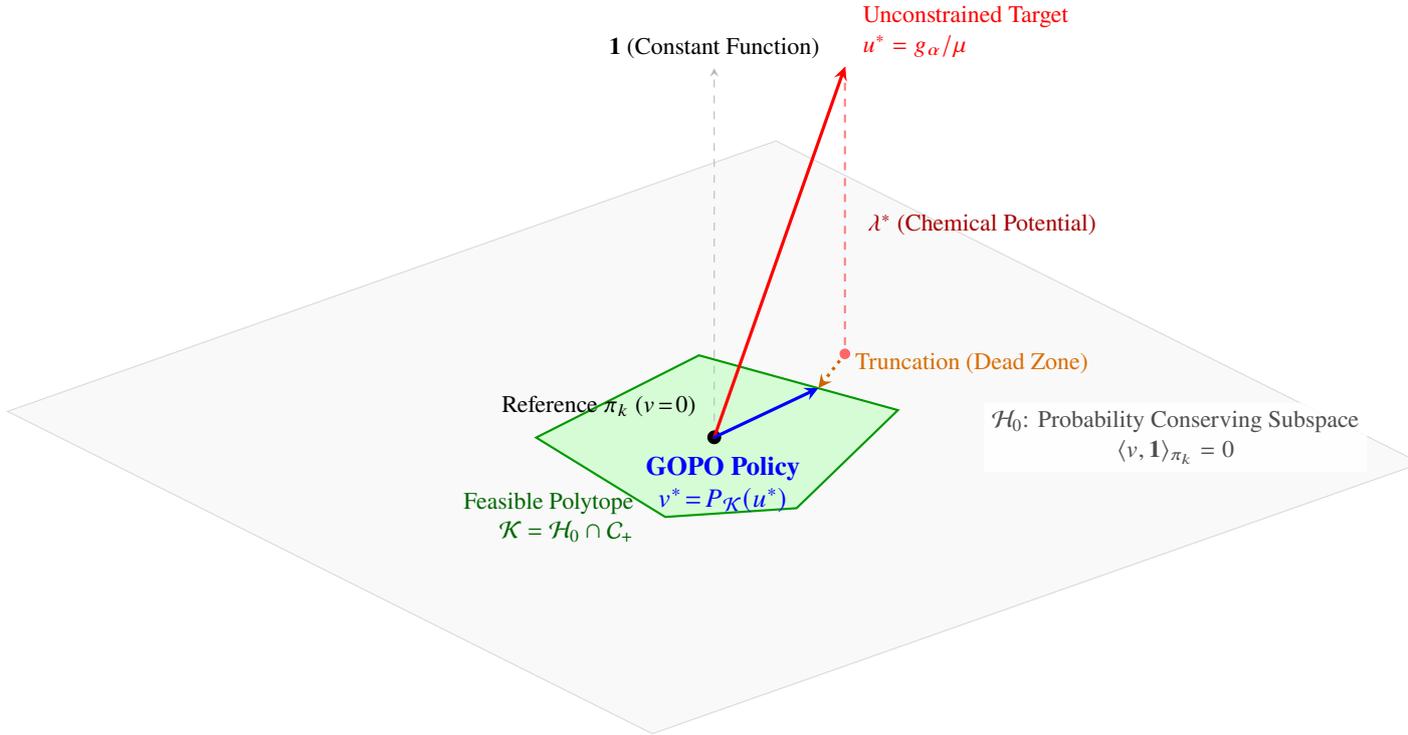
\begin{figure}[htbp]
\centering
\tdplotsetmaincoords{65}{50}

\begin{tikzpicture}[
    tdplot_main_coords,
    scale=3.0,
    >=stealth,
    label_font/.style={font=\footnotesize},
    important_label/.style={font=\small\bfseries},
    callout/.style={thin, gray!50}
]

    \coordinate (O) at (0,0,0);


    \fill[gray!10, opacity=0.45] (-2.2,-2.2,0) -- (2.2,-2.2,0) -- (2.2,2.2,0) -- (-2.2,2.2,0) -- cycle;
    \draw[gray!30, thin] (-2.2,-2.2,0) -- (2.2,-2.2,0) -- (2.2,2.2,0) -- (-2.2,2.2,0) -- cycle;

    \coordinate (K1) at (0.8, -0.2, 0);
    \coordinate (K2) at (0.3, 0.8, 0);
    \coordinate (K3) at (-0.7, 0.5, 0);
    \coordinate (K4) at (-0.5, -0.6, 0);
    \coordinate (K5) at (0.5, -0.7, 0);

    \fill[green!25, opacity=0.6] (K1) -- (K2) -- (K3) -- (K4) -- (K5) -- cycle;
    \draw[green!60!black, thick] (K1) -- (K2) -- (K3) -- (K4) -- (K5) -- cycle;


    \draw[->, gray!55, dashed, thin] (O) -- (0,0,1.8)
        node[above, black, label_font] {$\mathbf{1}$ (Constant Function)};
    \filldraw[black] (O) circle (0.8pt);

    \coordinate (U_star) at (-0.3, 1.0, 1.4);
    \coordinate (U_proj) at (-0.3, 1.0, 0);
    \coordinate (V_star) at (-0.1, 0.68, 0);

    \draw[dashed, red!50, line width=0.8pt] (U_star) -- (U_proj);
    \draw[->, dotted, orange!80!black, thick, line width=1.2pt] (U_proj) -- (V_star);
    \filldraw[red!60] (U_proj) circle (0.6pt);

    \draw[->, very thick, red, line cap=round] (O) -- (U_star);
    \draw[->, very thick, blue, line cap=round] (O) -- (V_star);


    \node[label_font, red, align=left, anchor=south west, xshift=3pt] at (U_star) {
        Unconstrained Target\\$u^* = g_\alpha/\mu$
    };

    \node[label_font, red!65!black, anchor=west, xshift=5pt] at ($(U_star)!0.55!(U_proj)$) {
        $\lambda^*$ (Chemical Potential)
    };

    \node[label_font, orange!85!black, anchor=west, xshift=5pt, yshift=3pt] at ($(U_proj)!0.5!(V_star)$) {
        Truncation (Dead Zone)
    };

    \node[important_label, blue, align=center, anchor=north east, xshift=-3pt, yshift=-22pt] at (V_star) {
        GOPO Policy\\$v^*\!=\!P_{\mathcal{K}}(u^*)$
    };

    \node[label_font, anchor=south east, xshift=-3pt, yshift=4pt] at (O) {Reference $\pi_k$ ($v\!=\!0$)};

    \node[label_font, green!40!black, anchor=east, align=right, xshift=-8pt] at (K5) {
        Feasible Polytope\\$\mathcal{K} = \mathcal{H}_0 \cap \mathcal{C}_+$
    };

    \node[label_font, gray!60!black, align=center,
          fill=white, fill opacity=0.85, text opacity=1,
          inner sep=3pt, rounded corners=1pt] at (1.3, 1.55, 0) {
        $\mathcal{H}_0$: Probability Conserving Subspace\\
        $\langle v, \mathbf{1} \rangle_{\pi_k} = 0$
    };

\end{tikzpicture}

\caption{\textbf{Geometric interpretation of GOPO.} The theory operates in $L^2(\pi_k)$. The reference policy $\pi_k$ sits at the origin ($v=0$). Valid policies must reside in $\mathcal{H}_0$ (gray plane) and satisfy non-negativity, restricting them to the feasible polytope $\mathcal{K}$ (green). The unconstrained target $u^* = g_\alpha/\mu$ (red) is first projected vertically onto $\mathcal{H}_0$ by subtracting the chemical potential $\lambda^*$, then truncated along the plane onto the $\mathcal{K}$ boundary (dead zone), yielding the final bounded GOPO update $v^*$ (blue).}
\label{fig:gopo_geometry}
\end{figure}

\subsection{From Global to Empirical: The GOPO Objective}
\label{subsec:gopo}

Optimizing the distance minimization~\eqref{eq:min_distance} in the global space $L^2(\pik)$ requires access to the full distribution, which is computationally intractable for autoregressive language models. GOPO bridges this gap by transitioning the orthogonal projection to an empirical discrete subspace.

For a given prompt $x$, sample a group of $G$ responses $\mathcal{G} = \{y_1, \ldots, y_G\} \sim \pik(\cdot|x)$. Define the empirical inner product space $\hat{L}^2(\mathcal{G})$ with metric:
\begin{equation}
\label{eq:empirical_ip}
    \langle f, g \rangle_{\mathcal{G}} = \frac{1}{G}\sum_{i=1}^G f(y_i)g(y_i)
\end{equation}
The probability conservation constraint is naturally relaxed to the empirical zero-mean condition:
\begin{equation}
    \hat{\mathcal{H}}_0 = \left\{ v \in \hat{L}^2(\mathcal{G}) : \frac{1}{G}\sum_{i=1}^G v(y_i) = 0 \right\}
\end{equation}

\paragraph{The Vanishing Chemical Potential.}
In standard RLHF pipelines, raw rewards are normalized into group-relative advantages $A_i = r_i - \bar{r}$, satisfying $\sum_{i=1}^G A_i = 0$ by construction. Consequently, the empirical driving vector $\bm{A} = [A_1, \ldots, A_G]^\top$ intrinsically resides in $\hat{\mathcal{H}}_0$. Because the driving force is already orthogonal to $\mathbf{1}$, the projection operator $P_{\hat{\mathcal{H}}_0}$ acts as the identity on $\bm{A}$:
\begin{equation}
\label{eq:vanishing_mu}
    P_{\hat{\mathcal{H}}_0}(\bm{A}) = \bm{A} - \frac{1}{G}\sum_{i=1}^G A_i \cdot \mathbf{1} = \bm{A}
\end{equation}
The chemical potential $\lambda^*$ required to enforce probability conservation \emph{exactly vanishes} at the empirical group level.

\begin{remark}[Structural Simplification]
This vanishing is not a numerical approximation but a \emph{structural consequence} of group normalization. It eliminates the need to solve for $\lambda^*$, collapsing the constrained projection into an unconstrained optimization problem.
\end{remark}

\paragraph{The GOPO Loss Function.}
Substituting the empirical advantage $A_i$ for the driving field $g_\alpha$, and replacing the analytical fluctuation $v(y_i)$ with the parameterized density ratio $\rho_\theta(y_i) = \pi_\theta(y_i|x)/\pik(y_i|x)$, the constrained functional maximization~\eqref{eq:constrained_opt} reduces to:
\begin{equation}
\label{eq:gopo_loss}
\boxed{
    \loss_{\text{GOPO}}(\theta) = -\E_{x \sim \mathcal{D}} \left[ \frac{1}{G}\sum_{i=1}^G \left[ A_i \rho_\theta(y_i) - \frac{\mu}{2}\big(\rho_\theta(y_i) - 1\big)^2 \right] \right]
}
\end{equation}
\paragraph{Structural Comparison with GRPO.}
It is instructive to contrast~\eqref{eq:gopo_loss} with the GRPO loss~\cite{Shao2024GRPO}:
\begin{equation}
\label{eq:grpo_loss}
    \loss_{\text{GRPO}}(\theta) = -\E_{x \sim \mathcal{D}}\left[\frac{1}{G}\sum_{i=1}^G \frac{1}{|y_i|}\sum_{t=1}^{|y_i|} \min\!\Big(\rho_\theta(y_{i,t}) \hat{A}_i,\; \text{clip}(\rho_\theta(y_{i,t}), 1{-}\epsilon, 1{+}\epsilon) \hat{A}_i\Big) - \beta D_{\KL}(\pi_\theta \| \pik)\right]
\end{equation}
where $\hat{A}_i$ denotes the standardized advantage, $\rho_\theta(y_{i,t})$ is the token-level density ratio, and $\epsilon$ is the clip radius. The two losses share the same fundamental motivation---optimizing a policy via group-sampled responses and normalized advantages---but differ in three fundamental ways:
\begin{enumerate}[leftmargin=2em]
    \item \textbf{Trust region mechanism.} GRPO enforces the trust region via \emph{hard clipping} of the token-level ratio $\rho_\theta(y_{i,t})$ alongside an explicit KL penalty, which introduces a non-smooth, piecewise-linear landscape with zero-gradient plateaus wherever $|\rho_\theta - 1| > \epsilon$. GOPO replaces this with a smooth \emph{quadratic penalty} $\frac{\mu}{2}(\rho_\theta - 1)^2$ on the sequence-level ratio that continuously penalizes deviation from the reference---the ratio can move freely while the restoring force grows linearly, never producing a gradient dead zone.
    \item \textbf{Optimization curvature.} The Hessian of GRPO's clipped surrogate is either $0$ (inside the flat clipped region) or undefined (at the clip boundary); its effective curvature is data- and state-dependent. GOPO's Hessian is the constant $\mu$, yielding a uniformly convex landscape with predictable convergence (\cref{thm:constant_curvature}).
    \item \textbf{Gradient behavior at convergence.} As $\rho_\theta \to 1$ and the policy approaches the reference, the GRPO gradient is dominated by the advantage-weighted policy gradient $\hat{A}_i \nabla_\theta \log\pi_\theta$, which scales with the advantage magnitude but does not account for the distance already traveled. GOPO's gradient $-A_i + \mu(\rho_\theta - 1)$ is proportional to the \emph{displacement from equilibrium}, providing a natural deceleration as the policy approaches its target $\rho^* = 1 + A_i/\mu$.
\end{enumerate}
In summary, GOPO trades GRPO's hard combinatorial constraint (clip) for a soft geometric one (projection), gaining smoothness, constant curvature, and a principled equilibrium structure at no additional computational cost.

\paragraph{Bounded GOPO Loss.}
The BHP truncation logic is absorbed into the loss bounds via:
\begin{equation}
\label{eq:gopo_bounded}
    \loss_{\text{GOPO}}^{\text{BHP}}(\theta) = \frac{1}{G}\sum_{i=1}^G \max\!\left(0, \; -A_i\rho_\theta(y_i) + \frac{\mu}{2}\big(\rho_\theta(y_i) - 1\big)^2 \right)
\end{equation}
where the $\max(0, \cdot)$ acts as a ReLU-like floor, halting gradients once $\rho_\theta(y_i) \to 0$ for highly negative advantages.

\section{Theoretical Analysis}
\label{sec:analysis}

\subsection{Constant Curvature and Decoupled Geometry}
\label{subsec:curvature}

The primary vulnerability of PPO, DPO, and GRPO is their reliance on the exponential divergence $D_{\KL}$. The Hessian of KL-based objectives scales non-linearly with the policy output, causing gradient saturation that must be mitigated via heuristic clipping.

\begin{theorem}[Constant Curvature Optimization]
\label{thm:constant_curvature}
Let $\ell(\rho) = -A\rho + \frac{\mu}{2}(\rho - 1)^2$ be the pointwise un-truncated GOPO loss. The Hessian with respect to the density ratio $\rho$ is the constant scalar $\mu$, independent of the advantage signal $A$, the current policy state, or the data distribution.
\end{theorem}

\begin{proof}
$\nabla_\rho \ell = -A + \mu(\rho - 1)$, hence $\nabla^2_\rho \ell = \mu$.
\end{proof}

This constant curvature has immediate consequences for the decoupling of design axes:

\begin{corollary}[Structural Decoupling]
\label{cor:decoupling}
In the GOPO objective, the first-order driving force depends on the advantage signal:
\begin{equation}
    \nabla_\rho \ell = -A + \mu(\rho - 1)
\end{equation}
while the second-order curvature $\nabla^2_\rho \ell = \mu$ is independent of $A$. Consequently, changing the advantage weighting (sampling geometry) does not alter the optimization landscape curvature (optimization geometry), and vice versa.
\end{corollary}

\begin{remark}[Contrast with KL-Based Methods]
In DPO, the effective loss is $\ell_{\text{DPO}} \approx -\log\sigma(\beta m)$, where $m$ is the logit margin. The local curvature $\beta^2\sigma(m)(1-\sigma(m))$ depends on both the temperature $\beta$ and the current margin $m$. Changing $\beta$ simultaneously alters the gradient magnitude \emph{and} the stability profile. This data-dependent curvature is the root cause of gradient saturation.
\end{remark}

\begin{corollary}[Global Linear Convergence]
\label{cor:convergence}
Gradient descent in ratio space on the GOPO objective follows the linear system:
\begin{equation}
    \rho_{k+1} - \rho^* = (1 - \eta\mu)(\rho_k - \rho^*)
\end{equation}
where $\rho^* = 1 + A/\mu$. For step size $0 < \eta < 2/\mu$, this exhibits global contraction to the unique equilibrium at rate $|1 - \eta\mu|$, independent of the advantage distribution.
\end{corollary}

\subsection{Gradient Dynamics: Non-Saturation Guarantee}
\label{subsec:gradient}

We quantify the gradient behavior to contrast with KL-based methods.

\paragraph{Gradient Saturation in DPO.}
For logistic losses, the gradient magnitude satisfies $|\nabla_m \ell_{\text{DPO}}| = |\beta\sigma(m)(1-\sigma(m))| \leq \beta/4$, and crucially, $|\nabla_m \ell_{\text{DPO}}| \to 0$ exponentially as $|m| \to \infty$.

\paragraph{Non-Saturating Gradient in GOPO.}
For the GOPO loss, the gradient magnitude is:
\begin{equation}
    |\nabla_\rho \ell_{\text{GOPO}}| = |{-A + \mu(\rho - 1)}| = \mu|\rho - \rho^*|
\end{equation}
This is exactly proportional to the distance from equilibrium. For any non-equilibrium state with $|\rho - \rho^*| \geq \delta$:
\begin{equation}
    |\nabla_{\text{GOPO}}| \geq \mu\delta
\end{equation}
The gradient maintains a \emph{linear} driving force that never vanishes away from equilibrium, preventing the saturation endemic to KL-based methods.

\subsection{Dead Zone Dynamics and Gradient Hard Stops}
\label{subsec:dead_zone}

The BHP truncation establishes an implicit ``dead zone'' for gradient flow. Taking the derivative of the bounded GOPO loss~\eqref{eq:gopo_bounded}:
\begin{equation}
\label{eq:dead_zone_grad}
    \nabla_\theta \loss_{\text{GOPO}}^{\text{BHP}} \propto \frac{1}{G}\sum_{i=1}^G \underbrace{\mathbb{I}\!\big(\rho_\theta(y_i) > 0\big)}_{\text{Dead Zone Gate}} \cdot \underbrace{\big(-A_i + \mu(\rho_\theta(y_i) - 1)\big)}_{\text{Restoring Force}} \nabla_\theta \rho_\theta(y_i)
\end{equation}
When an action receives an extremely negative advantage $A_i \ll 0$, the ratio $\rho_\theta(y_i)$ is rapidly suppressed. Once it approaches the physical boundary $\rho_\theta \to 0$ (equivalently $v(y_i) \to -1$), the indicator function triggers a \emph{hard stop}: the gradient for this action is exactly zeroed out.

This mechanism prevents the network from wasting representational capacity on already-suppressed outputs, redirecting gradient bandwidth exclusively toward distinguishing among higher-quality responses.

\subsection{Connection to $\chi^2$ Divergence}
\label{subsec:chi2}

The Hilbert space framework provides a natural connection to statistical divergences.

\begin{proposition}[$\chi^2$ Interpretation]
\label{prop:chi2}
The dissipation term in the GOPO functional is proportional to the Pearson $\chi^2$ divergence:
\begin{equation}
    \frac{1}{2}\E_{\pik}\!\left[(v(y))^2\right] = \frac{1}{2}\E_{\pik}\!\left[\left(\frac{\pi(y)}{\pik(y)} - 1\right)^2\right] = D_{\chi^2}(\pi \| \pik)
\end{equation}
Thus, the GOPO objective can be equivalently read as maximizing advantage-weighted work subject to a $\chi^2$ trust region.
\end{proposition}

\begin{proposition}[Total Variation Bound]
\label{prop:tv}
Bounding $\E[v^2] \leq \epsilon$ via the $\chi^2$ penalty guarantees, by Jensen's inequality:
\begin{equation}
    \mathrm{TV}(\pi, \pik) = \frac{1}{2}\E[|v|] \leq \frac{1}{2}\sqrt{\E[v^2]} \leq \frac{\sqrt{\epsilon}}{2}
\end{equation}
providing a distributional stability guarantee in Total Variation distance.
\end{proposition}

\subsection{Comparison with Parametric L2 Regularization}
\label{subsec:l2}

A natural question is whether GOPO's quadratic penalty differs from standard weight decay $\frac{\lambda}{2}\|\theta - \theta_0\|^2$. The distinction is fundamental:

\begin{enumerate}[leftmargin=2em]
    \item \textbf{Functional vs. Parametric}: GOPO's penalty $\frac{\mu}{2}\E[v^2]$ acts in the function space of probability ratios. It penalizes deviations where the policy assigns probability mass. Weight decay acts uniformly in parameter space, blind to output-level consequences.
    \item \textbf{Adaptive Restoring Force}: The gradient $\mu(\rho - 1)$ pulls each specific output ratio back toward unity. Weight decay suppresses all weights uniformly, potentially hindering the formation of sharp, high-confidence reasoning chains.
    \item \textbf{Geometric Guarantee}: The $\chi^2$ penalty provides distributional bounds (\cref{prop:tv}). Weight decay offers no such output-space guarantee.
\end{enumerate}

\section{Algorithm and Implementation}
\label{sec:algorithm}

\begin{figure}[htbp]
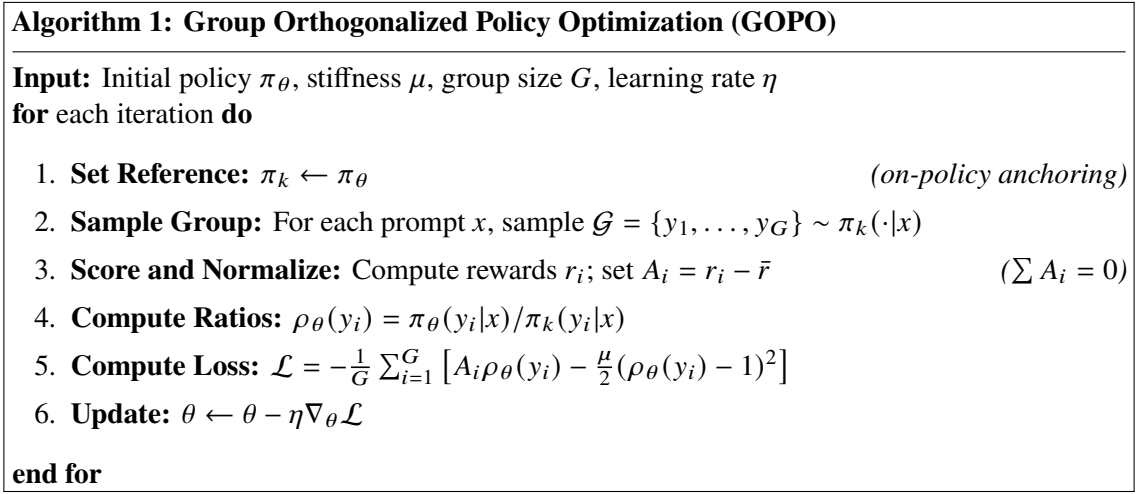

\centering
\fbox{\parbox{0.92\textwidth}{
\textbf{Algorithm 1: Group Orthogonalized Policy Optimization (GOPO)}
\vspace{0.5em}
\hrule
\vspace{0.5em}
\textbf{Input:} Initial policy $\pi_\theta$, stiffness $\mu$, group size $G$, learning rate $\eta$

\textbf{for} each iteration \textbf{do}
\begin{enumerate}[leftmargin=2em, itemsep=0pt]
    \item \textbf{Set Reference:} $\pik \leftarrow \pi_\theta$ \hfill \textit{(on-policy anchoring)}
    \item \textbf{Sample Group:} For each prompt $x$, sample $\mathcal{G} = \{y_1, \ldots, y_G\} \sim \pik(\cdot|x)$
    \item \textbf{Score and Normalize:} Compute rewards $r_i$; set $A_i = r_i - \bar{r}$ \hfill \textit{($\sum A_i = 0$)}
    \item \textbf{Compute Ratios:} $\rho_\theta(y_i) = \pi_\theta(y_i|x) / \pik(y_i|x)$
    \item \textbf{Compute Loss:} $\loss = -\frac{1}{G}\sum_{i=1}^G \big[A_i\rho_\theta(y_i) - \frac{\mu}{2}(\rho_\theta(y_i) - 1)^2\big]$
    \item \textbf{Update:} $\theta \leftarrow \theta - \eta\nabla_\theta \loss$
\end{enumerate}
\textbf{end for}
}}
\caption{GOPO algorithm. Step 3 guarantees that the advantage vector lies in the zero-mean subspace $\hat{\mathcal{H}}_0$, eliminating the chemical potential. Steps 4--5 implement the empirical orthogonal projection with quadratic dissipation.}
\label{alg:gopo}
\end{figure}

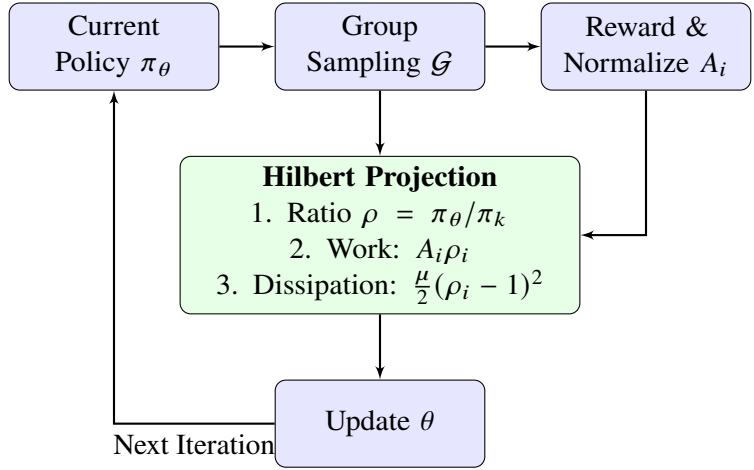
\begin{figure}[htbp]
\centering
\begin{tikzpicture}[
    node distance=1.5cm,
    auto,
    block/.style={
      rectangle,
      draw,
      fill=blue!10,
      text width=6.5em,
      text centered,
      rounded corners,
      minimum height=3em
    },
    line/.style={draw, -latex', thick},
]
    \node [block] (policy) {Current Policy $\pi_\theta$};
    \node [block, right of=policy, node distance=3.5cm] (sampling) {Group Sampling $\mathcal{G}$};
    \node [block, right of=sampling, node distance=3.5cm] (eval) {Reward \& Normalize $A_i$};
    
    \node [block, below of=sampling, node distance=2.5cm, fill=green!10, text width=13em] (gopo) {\textbf{Hilbert Projection}\\ 1. Ratio $\rho = \pi_\theta/\pik$ \\ 2. Work: $A_i \rho_i$ \\ 3. Dissipation: $\frac{\mu}{2}(\rho_i - 1)^2$};
    
    \node [block, below of=gopo, node distance=2.5cm] (update) {Update $\theta$};
    
    \path [line] (policy) -- (sampling);
    \path [line] (sampling) -- (eval);
    \path [line] (eval) |- (gopo);
    \path [line] (sampling) -- (gopo);
    \path [line] (gopo) -- (update);
    \path [line] (update) -| node [near start] {Next Iteration} (policy);
    
\end{tikzpicture}
\caption{GOPO flowchart. The core operation (green) implements the empirical orthogonal projection: the advantage signal provides the linear driving force, while the quadratic dissipation term provides the constant-curvature regularization.}
\label{fig:flowchart}
\end{figure}

\paragraph{Implementation Notes.}
GOPO introduces no additional model forward/backward passes compared to GRPO. The only overhead is the computation of $(\rho_\theta - 1)^2$, which is negligible. The reference policy $\pik$ is set to the policy from the start of each iteration (on-policy anchoring), ensuring that the trust region $\rho \approx 1$ holds locally. No critic network is required.

\section{Experiments}
\label{sec:experiments}

We evaluate GOPO against four strong baselines on mathematical reasoning tasks, focusing on out-of-distribution generalization performance.

\subsection{Setup}

\paragraph{Model and Data.} We use Qwen3-1.7B as the base model, trained with the VERL framework~\cite{Sheng2024HybridFlow} on 4$\times$ RTX 4090 GPUs. To stress-test sample efficiency and generalization, we deliberately use a small training set: approximately 10\% of MATH Level 3 problems (sampled with seed 42). Validation is performed every 10 training steps on 100 held-out MATH Level 4 problems---a strictly harder difficulty tier not seen during training---to measure out-of-distribution generalization. All methods share identical hyperparameters: batch size 48, learning rate $2 \times 10^{-6}$, 8 epochs over the training data, and $G=6$ rollout generations per prompt.

\paragraph{Baselines.}
\begin{itemize}[leftmargin=2em]
    \item \textbf{OPO}~\cite{Wang2025OPO}: The full Orthogonalized Policy Optimization with importance-sampling reweighting ($\omega_\alpha$), from which GOPO is derived.
    \item \textbf{GRPO}~\cite{Shao2024GRPO}: Token-level policy gradient with group-normalized advantages and PPO-style ratio clipping ($\epsilon = 0.2$).
    \item \textbf{GSPO}: Sentence-level variant of GRPO with step-level advantage normalization.
    \item \textbf{DAPO}~\cite{Yu2025DAPO}: Advanced baseline with asymmetric clip bounds, unnormalized advantages, and overlong reward shaping.
\end{itemize}
GOPO uses stiffness $\mu = 0.5$, escort exponent $\alpha = 0.5$, and on-policy anchoring.

\subsection{Results}

\paragraph{Overall Performance.} \cref{tab:results} summarizes the final metrics for all methods.

\begin{table}[htbp]
\centering
\caption{Comparison of policy optimization algorithms on MATH benchmarks. All methods use Qwen3-1.7B trained on $\sim$10\% of MATH Level 3, validated on 100 MATH Level 4 problems. Reward is averaged over the entire training process.}
\label{tab:results}
\begin{tabular}{lcccc}
\toprule
\textbf{Algorithm} & \textbf{Mean Reward}$\uparrow$ & \textbf{Val Acc (L4)}$\uparrow$ & \textbf{Grad Norm} & \textbf{Entropy} \\
\midrule
GRPO & 0.544 & 44\% & 0.674 & 0.115 \\
DAPO & 0.548 & 44\% & 0.213 & 0.126 \\
GSPO & 0.553 & \textbf{48\%} & 0.623 & 0.128 \\
OPO  & \textbf{0.558} & \textbf{48\%} & 1.279 & 0.126 \\
\textbf{GOPO (Ours)} & 0.555 & 47\% & 1.029 & \textbf{0.134} \\
\bottomrule
\end{tabular}
\end{table}

\paragraph{Analysis.}
\begin{itemize}[leftmargin=2em]
    \item \textbf{Generalization vs.\ Training Reward.} OPO and GOPO achieve the highest mean rewards (0.558 and 0.555) over the training process, which directly translates to their superior generalization performance (48\% and 47\% validation accuracy). In contrast, GRPO struggles with the lowest mean reward (0.544) and poorest generalization (44\%), struggling to efficiently learn from the small training set. GOPO achieves competitive validation accuracy while maintaining an upward trajectory (\cref{fig:training_curves}b).
    \item \textbf{Gradient Dynamics.} OPO and GOPO maintain substantially higher gradient norms (1.28 and 1.03) than ratio-clipped methods (GRPO: 0.67, GSPO: 0.62, DAPO: 0.21), empirically confirming the non-saturation prediction of \cref{thm:constant_curvature}. DAPO's severely diminished gradients ($\sim$0.21) indicate the conservative clipping is overly restrictive.
    \item \textbf{Entropy Preservation.} GOPO maintains the highest policy entropy (0.134) among all methods. We attribute this to the sequence-level nature of the GOPO update: unlike token-level methods (GRPO, DAPO) that impose dense per-token supervision forcing rapid mode collapse, GOPO aligns total trajectory probability without micromanaging individual tokens.
\end{itemize}

\paragraph{Training Dynamics.} \cref{fig:training_curves} shows the full training trajectories across four diagnostic metrics.

\begin{figure}[htbp]
\centering
\includegraphics[width=\textwidth]{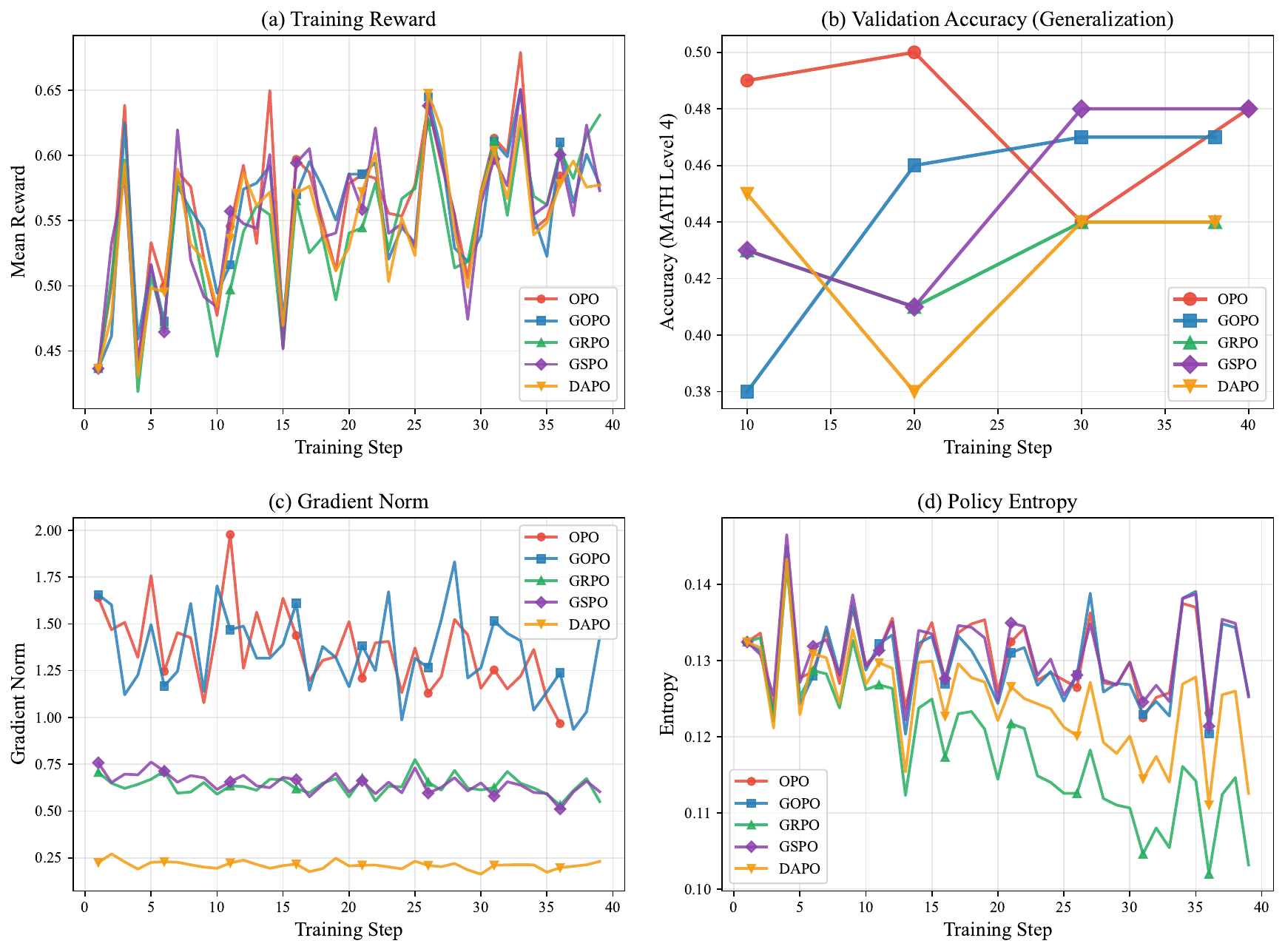}
\caption{\textbf{Training dynamics comparison.} (a) Training reward: OPO and GOPO achieve the highest mean rewards over the training process. (b) Validation accuracy on MATH Level 4: OPO and GSPO lead at 48\%, with GOPO at 47\% and monotonically improving. (c) Gradient norm: OPO/GOPO maintain healthy norms throughout, while DAPO exhibits severe gradient saturation. (d) Policy entropy: GOPO preserves the most diversity, preventing premature mode collapse.}
\label{fig:training_curves}
\end{figure}

\paragraph{Generalization Analysis.} The validation accuracy trajectories (\cref{fig:training_curves}b) reveal important dynamics:
\begin{itemize}[leftmargin=2em]
    \item GOPO shows \emph{monotonic improvement} in generalization (38\% $\to$ 46\% $\to$ 47\%), suggesting the algorithm has not yet plateaued and would benefit from longer training.
    \item OPO starts strong (49\% at step 10) but exhibits non-monotonic behavior (50\% $\to$ 44\% $\to$ 48\%), consistent with the importance-sampling variance inherent in the full OPO framework.
    \item GRPO and DAPO plateau at 44\%, despite GRPO's higher training reward---a signature of the reward hacking phenomenon where clipping-based methods exploit the training distribution rather than learning generalizable reasoning patterns.
\end{itemize}

\section{Discussion}
\label{sec:discussion}

\paragraph{Why Hilbert Space?}
Our framework is not merely a notational change but a deliberate choice of the mathematical arena best suited to the problem's intrinsic structure. Policy alignment fundamentally operates in an abstract, high-dimensional space where the action vocabulary can range from tens of thousands of tokens to unbounded sequences. In finite-dimensional Euclidean space, optimization over such objects requires explicit coordinates and encounters the curse of dimensionality. The Hilbert space $L^2(\pik)$, by contrast, is purpose-built for abstract infinite-dimensional analysis: it provides a complete inner product structure that makes distances, angles, and projections well-defined \emph{regardless of the underlying dimensionality}. Consequently, mappings and constraints that appear intractably nonlinear on the probability simplex---such as the normalization condition $\sum_y \pi(y) = 1$---collapse into elementary linear-algebraic operations (inner products, closed subspaces, orthogonal projections) once lifted into this space. The probability simplex, equipped with KL geometry, fundamentally entangles the sampling signal with the optimization curvature; the Hilbert space disentangles them entirely. The Hilbert Projection Theorem provides a principled replacement for heuristic clipping, and the chemical potential emerges as a natural geometric by-product of the projection rather than an ad-hoc normalization constant.

\paragraph{Abstract Representations and the Geometry of Distance.}
The adoption of Hilbert space is not without precedent in sciences that confront fundamentally abstract, high-dimensional state spaces. In quantum mechanics, the state of a physical system is a vector in a Hilbert space, and all measurable predictions---probabilities, expectation values, transition amplitudes---reduce to inner products and projections within that space~\cite{VonNeumann1932}. The power of this formalism lies precisely in its \emph{coordinate-free} nature: the mathematical apparatus operates identically whether the state space is two-dimensional (a qubit) or infinite-dimensional (a quantum field), because the inner product provides a complete metric structure---distance, angle, and projection---independent of any particular basis.

A structurally parallel situation arises in language model alignment. The output of an LLM is a probability distribution over an astronomically large discrete space of possible utterances. What we seek to optimize is not any individual coordinate of this distribution, but its \emph{global geometric relationship} to a target: how ``close'' the aligned policy is to the ideal, constrained by the requirement of remaining a valid distribution. The $L^2(\pik)$ framework provides exactly the right notion of distance ($\|v - u^*\|_{\pik}$) for this task---one that is complete, geometrically meaningful, and agnostic to the dimensionality of the action space. The alignment problem thus reduces to its purest geometric essence: \emph{find the nearest valid point to an ideal target}, an operation that Hilbert space is uniquely equipped to perform via its Projection Theorem.

This perspective also connects to a broader trend in computational cognitive science, where abstract representational spaces---from semantic embeddings to conceptual structures---are increasingly modeled as inner product spaces in which ``similarity'' corresponds to geometric proximity~\cite{Gardenfors2000,Busemeyer2012}. The Hilbert space framework may thus offer a principled mathematical language not only for policy optimization but for any domain where the fundamental task is to navigate distances between abstract, high-dimensional representations under structural constraints.

\paragraph{The Role of Group Normalization.}
The vanishing chemical potential at the group level is a structural consequence of the geometry. Group normalization ensures that the advantage vector already lies in the zero-mean subspace $\hat{\mathcal{H}}_0$, so no projection is needed along the $\mathbf{1}$ direction. This eliminates an entire class of potential failure modes related to improper constraint enforcement.

\paragraph{Relation to Existing Methods.}
\begin{itemize}[leftmargin=2em]
    \item Setting $\mu \to 0$ recovers an unconstrained policy gradient (no regularization).
    \item GOPO's quadratic penalty can be viewed as the Bregman divergence induced by the Euclidean mirror map $\Psi(v) = \frac{1}{2}\|v\|^2$, connecting to the mirror descent literature. However, the Hilbert space derivation is more direct and reveals additional structure (subspaces, projections, chemical potential).
    \item The dead-zone mechanism (\cref{subsec:dead_zone}) provides a principled alternative to the clip-higher strategy of DAPO~\cite{Yu2025DAPO}, with the advantage of arising from the geometry rather than being hand-designed.
\end{itemize}

\paragraph{Limitations.}
GOPO introduces the stiffness parameter $\mu$, which may require tuning. Our experiments focus on mathematical reasoning; validation on diverse domains (instruction following, code generation) remains future work. The BHP truncation is currently implemented via a soft ReLU approximation rather than the exact hard threshold derived in \cref{thm:bhp}; studying the impact of this approximation gap is an open direction.

\section{Conclusion}
\label{sec:conclusion}

We have presented \textbf{Group Orthogonalized Policy Optimization (GOPO)}, an alignment algorithm derived entirely from the geometry of Hilbert function spaces. By lifting policy optimization from the probability simplex into $L^2(\pik)$, we transform the nonlinear normalization constraint into a linear orthogonality condition. The alignment objective itself emerges algebraically from the geometric principle of minimum distance to the unconstrained target, and its closed-form solution is given by the Hilbert Projection Theorem. The Bounded Hilbert Projection extends this to enforce non-negativity, yielding exact sparsity for catastrophic outputs. At the empirical group level, the chemical potential vanishes under standard advantage normalization, producing a simple loss with constant curvature, non-saturating gradients, and an intrinsic dead-zone mechanism. Experiments on mathematical reasoning benchmarks confirm that GOPO sustains learning in high-confidence regimes where clipping-based methods plateau, achieving competitive generalization performance with the healthiest gradient dynamics and entropy preservation among all tested methods.


\appendix
\section{Proofs and Derivations}
\label{app:proofs}

\subsection{Derivation of the Projection Operator}
\label{app:projection}

\begin{proof}[Proof of \cref{thm:projection}]
The Hilbert space $\mathcal{H} = L^2(\pik)$ admits the orthogonal decomposition $\mathcal{H} = \mathcal{H}_0 \oplus \text{span}\{\mathbf{1}\}$. For any $f \in \mathcal{H}$, the projection onto $\mathcal{H}_0$ is:
\begin{equation}
    P_{\mathcal{H}_0}(f) = f - \frac{\langle f, \mathbf{1}\rangle_{\pik}}{\langle \mathbf{1}, \mathbf{1}\rangle_{\pik}}\mathbf{1} = f - \E_{\pik}[f]
\end{equation}
As derived in \cref{subsec:emergence}, minimizing $\frac{\mu}{2}\|v - u^*\|^2$ over $v \in \mathcal{H}_0$ is equivalent to maximizing $\mathcal{J}(v) = \langle g_\alpha, v\rangle - \frac{\mu}{2}\|v\|^2$, since the two differ only by the constant $\frac{1}{2\mu}\|g_\alpha\|^2$. The minimum-distance problem has the unique solution $v^* = P_{\mathcal{H}_0}(u^*) = P_{\mathcal{H}_0}(g_\alpha/\mu)$ by the Hilbert Projection Theorem. Applying the projection formula:
\begin{equation}
    v^* = \frac{g_\alpha}{\mu} - \E_{\pik}\!\left[\frac{g_\alpha}{\mu}\right] = \frac{1}{\mu}\big(g_\alpha - \E_{\pik}[g_\alpha]\big)
\end{equation}
\end{proof}

\subsection{Derivation of the Bounded Hilbert Projection}
\label{app:bhp}

\begin{proof}[Proof of \cref{thm:bhp}]
The KKT stationarity condition for the Lagrangian gives:
\begin{equation}
    g_\alpha(y) - \mu v(y) - \lambda + \eta(y) = 0
\end{equation}
Combined with the complementarity condition $\eta(y)(v(y) + 1) = 0$ and $\eta(y) \geq 0$:

\textbf{Case 1}: $v(y) > -1$. Then $\eta(y) = 0$, so $v(y) = (g_\alpha(y) - \lambda)/\mu$.

\textbf{Case 2}: $v(y) = -1$. Then $\eta(y) = \lambda - \mu - g_\alpha(y) \geq 0$, which requires $g_\alpha(y) \leq \lambda^* - \mu$.

These cases unify as $v^*(y) = \max(-1, (g_\alpha(y) - \lambda^*)/\mu)$, where $\lambda^*$ solves $\E_{\pik}[v^*] = 0$.
\end{proof}

\subsection{Log-Ratio Approximation}
\label{app:log_approx}

\begin{lemma}[Log-Ratio Approximation Error]
Let $\Delta_\theta(y) = \log\pi_\theta(y) - \log\pik(y)$ and $v_\theta(y) = \exp(\Delta_\theta) - 1$. For $\delta = \|\Delta_\theta\|_\infty < 1$:
\begin{equation}
    |v_\theta - \Delta_\theta| \leq \frac{1}{2}\delta^2 e^\delta, \quad |v_\theta^2 - \Delta_\theta^2| \leq \delta^3 e^{2\delta}
\end{equation}
The GOPO loss using log-ratios $\Delta_\theta$ in place of $v_\theta = \rho_\theta - 1$ incurs an approximation error of $O(\E[\Delta_\theta^3])$, which is controlled by the on-policy trust region.
\end{lemma}

\subsection{$\chi^2$-Constrained Maximization Duality}
\label{app:duality}

\begin{proposition}[Lagrange Dual Equivalence]
The constrained problem
$\max_{v \in L^2(\pik)} \E_{\pik}[g_\alpha v]$ subject to $\E_{\pik}[v^2] \leq \epsilon$
has Lagrangian relaxation equal to the GOPO functional $\mathcal{J}(v) = \E[g_\alpha v] - \frac{\mu}{2}\E[v^2]$, where $\mu$ is the dual variable corresponding to the constraint radius $\epsilon$. The optimal solution $v^* = g_\alpha/\mu$ confirms that the stiffness $\mu$ purely controls the trust-region radius while $g_\alpha$ purely shapes the alignment direction.
\end{proposition}

\end{document}